\documentclass[11pt]{article}
\usepackage[preprint]{acl}
\usepackage{times}
\usepackage{latexsym}
\usepackage{booktabs}
\usepackage{multirow}
\usepackage[T1]{fontenc}
\usepackage[utf8]{inputenc}
\usepackage{microtype}
\usepackage{inconsolata}
\usepackage{xurl}
\usepackage{xcolor}
\usepackage{graphicx}
\usepackage{amsmath}
\usepackage{amsthm}
\newtheorem{proposition}{Proposition} \newtheorem{lemma}[proposition]{Lemma} \newtheorem{corollary}[proposition]{Corollary} \theoremstyle{remark} \newtheorem{remark}[proposition]{Remark} \theoremstyle{definition} \newtheorem{definition}{Definition}
\usepackage{amssymb}
\usepackage{subcaption}
\usepackage{enumitem}
\usepackage{pifont}
\usepackage{float}
\usepackage{fontawesome5}

\newcommand{\Br}{B_{\mathrm{r}}}
\newcommand{\Bsimple}{B_{\mathrm{simple}}}
\newcommand{\Dauc}{D_{\mathrm{AUC}}}
\newcommand{\dclip}{d_{\mathrm{clip}}}

\title{Two Clocks: Why Larger Rollout Batches Learn Slower but Fail on Schedule in Reinforcement Learning with Verifiable Rewards}

\author{
Yuante Li\textsuperscript{\rm 1}\thanks{Equal contribution.},\hspace{4pt}
Junwei Quan\textsuperscript{\rm 2}\footnotemark[1],\hspace{4pt}
Kai Qiu\textsuperscript{\rm 1},\hspace{4pt}
Kate Zhang\textsuperscript{\rm 1},\hspace{4pt}
Yicheng Tao\textsuperscript{\rm 1}
\\
\textsuperscript{\rm 1} Carnegie Mellon University,\hspace{4pt}
\textsuperscript{\rm 2} University of Toronto
\\
\textcolor{black}{\faEnvelope}\  \texttt{yuantel@cs.cmu.edu, junwei.quan@mail.utoronto.ca}
}

\begin{document}
\maketitle
% abstract word limit 200 words
\begin{abstract}
Reinforcement learning with verifiable rewards (RLVR) relies on costly model-generated responses, motivating larger rollout batches and square-root learning-rate scaling.
Reported to approximately preserve progress per sample with one optimizer step per rollout, this rule is less understood in practical multi-step training.
We develop a two-clock framework that separates the samples needed to make progress from those consumed before collapse.
Under a local quadratic model, the square-root rule preserves first-order drift per sample but not progress: larger batches take fewer, longer updates and lose more to curvature, yielding a learning-rate-dependent critical batch beyond which sample efficiency degrades.
A drift-threshold hypothesis predicts a complementary failure clock whose sample count remains approximately invariant along the square-root rule.
Controlled GRPO experiments on Qwen2.5-Math-1.5B across a $16\times$ batch range support both predictions.
Ratio clipping delays failure, especially for larger batches, but the progress gap persists without it.
Widening the clip trades stability for a transient speed gain.
The progress gap reappears on another dataset, a general-purpose checkpoint, and a 7B model.
Finally, we use the two-clock framework to derive a simple multi-step RLVR schedule: start below the critical batch and switch to a larger batch before failure.
Our code is available at: \href{https://github.com/WinstonLiyt/two-clocks-rlvr}{two-clocks-rlvr}.
\end{abstract}

\section{Introduction}
\label{sec:intro}

\begin{figure}[t]
\centering
\includegraphics[width=\linewidth]{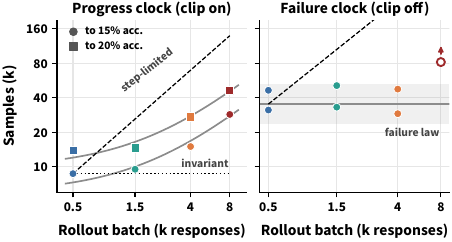}
\caption{The two clocks under the square-root rule with $\mu{=}8$.
Color marks the rollout batch, from 512 in blue to 8,192 in dark red.
Left, with the clip: samples to reach 15\% and 20\% accuracy, with the critical-batch curve (gray) between perfect invariance (dotted) and the step-limited extreme (dashed).
Right: samples to failure without the clip, against the failure-law band (gray), with an open marker for the run that never failed.
Progress slows with the batch, while failure stays on the failure law.}
\vspace{-2em}
\label{fig:clocks}
\end{figure}

Reinforcement learning with verifiable rewards (RLVR)~\citep{lambert2024tulu, guo2025deepseek} has become a widely used post-training approach for language models, spanning mathematical reasoning~\citep{hubert2026alphaproof, team2025kimi}, code~\citep{wei2026swe, luo2025deepcoder}, and broader domains scored against reference answers~\citep{su2025crossing}.
Its appeal is that large amounts of policy-generated data can be scored without human preference labels, but training remains expensive because this data must be regenerated as the policy changes: \textbf{every policy update requires fresh rollouts, and generation dominates the cost.} Practitioners therefore enlarge the \emph{rollout batch}, the number of responses generated by one policy before it is updated, and rely on the \emph{square-root rule} to transfer the learning rate across batch sizes, scaling it with the square root of the batch~\citep{malladi2022sde,hilton2022batch}.
For Group Relative Policy Optimization (GRPO)~\citep{shao2024deepseekmath} with one optimizer step per rollout, the rule approximately works: rollout batches from 512 to 8K responses follow learning curves within about one accuracy point of each other when plotted against the number of samples consumed~\citep{li2026larger}.

However, practical RLVR training rarely operates in this one-step regime.
To amortize generation, practical RLVR recipes split each rollout into mini-batches and take one optimizer step on each.
DAPO, for example, uses 16 optimizer steps per rollout~\citep{yu2026dapo}.
From the second step onward, the current policy has moved away from the policy that generated the data, so the token-level importance ratios depart from one and the PPO-style ratio clip begins to suppress gradient contributions from tokens whose ratios move beyond the clipping range.

In our multi-step GRPO experiments, three phenomena appear together under the square-root rule: \textbf{\ding{192} Progress slows with batch size.} Larger rollout batches learn more slowly per sample.
\textbf{\ding{193} Failure keeps time in samples.} Without the clip, larger batches collapse after the same number of samples as smaller ones, despite taking far fewer optimizer steps.
\textbf{\ding{194} Clipping helps larger batches most.} Although clipping is often viewed as a bottleneck on learning~\citep{qi2026rethinking,yang2026clipping}, its stabilizing effect is strongest for the larger batches.
Figure~\ref{fig:clocks} visualizes this split into two clocks: progress becomes increasingly sample-inefficient with larger rollout batches, while failure stays locked to the number of samples consumed.
Prior work studies these effects separately: batch-size invariance~\citep{hilton2022batch,li2026larger}, critical-batch behavior with the learning rate tuned at every batch~\citep{mccandlish2018empirical}, and clipping at a fixed batch size.
None explains why all three emerge together under the same batch-dependent learning-rate rule.

We explain these phenomena with a \textbf{two-clock framework} of batch scaling.
Its central observation is that \textit{the square-root rule preserves the policy's drift per sample, the persistent displacement it accumulates, but not the progress that drift produces.} A larger rollout batch takes fewer optimizer steps on the same samples, and the square-root rule makes each step longer.
Under a local quadratic model, these fewer, longer steps accumulate the same drift per sample but lose more progress to curvature.
This gives a \underline{\textbf{progress clock}}, which counts samples to a target accuracy and slows beyond a learning-rate-dependent critical batch.
If collapse occurs when accumulated drift reaches a threshold, the same scaling rule gives a \underline{\textbf{failure clock}} that is invariant in samples.
The progress clock follows from the local quadratic model, and the failure clock rests on a drift-threshold hypothesis that we test.

We test the two-clock model with controlled GRPO experiments that vary rollout batch size, learning rate, and clipping over a sixteenfold range of rollout batches.
The progress clock departs from sample invariance beyond a critical batch: with eight steps per rollout, as the learning rate rises tenfold, the sample-cost exponent climbs from 0.15 to 0.44.
Unlike the tuned-learning-rate picture~\citep{mccandlish2018empirical}, the critical batch therefore shrinks with the learning rate assigned by the square-root rule and lies well below the gradient noise scale.
The failure clock, in contrast, does not move along the rule: without the clip, a 4,096 batch fails after 56 to 80 optimizer steps and a 512 batch after 648, yet both fail after 30 to 50k samples.
Restoring the clip delays collapse twofold for the small batch and fourfold for the large one, and the progress gap is present with or without it.
The progress gap reappears on a second dataset, a second checkpoint, and a 7B model.
The separation between the two clocks further yields a simple batch schedule: \textbf{start below the critical batch for sample-efficient progress, then switch to a larger batch when throughput or stability makes it preferable.}
The ramp approaches the best fixed batch while generating three quarters of its samples at the large batch.

To summarize, our contributions are as follows.
\begin{itemize}
[nosep,leftmargin=12pt,topsep=0pt,itemsep=1pt,parsep=0pt]
\item[\ding{192}] We introduce a two-clock framework for multi-step GRPO, in which a local model derives a learning-rate-dependent critical batch for progress, while a drift-threshold hypothesis predicts failure at a fixed number of samples along
the square-root rule.
\item[\ding{193}] We test these predictions across batch size, learning rate, clipping, model scale, and dataset, with failure-law fits that favor drift scaling over noise-driven alternatives and clip and replay controls showing that the clip delays failure while the progress gap persists without it.
\item[\ding{194}] We derive a practical batch schedule for multi-step RLVR from the two clocks:
start below the critical batch and switch to a larger batch before failure. In our experiments, this schedule approaches the best fixed-batch performance.
\end{itemize}

\section{Preliminaries}
\label{sec:prelim}

\paragraph{GRPO with a ratio clip.} For each prompt, the policy samples $G$ responses and receives rule rewards $r_i\in\{0,1\}$, which GRPO normalizes within the group into advantages $A_i=(r_i-\bar r)/(\sigma_r+\epsilon)$~\citep{shao2024deepseekmath}.
For token $y_t$, let $\rho_t=\pi_\theta(y_t\mid\cdot)/\pi_{\theta_{\mathrm{old}}}(y_t\mid\cdot)$ be the importance ratio between the current policy and the policy that generated the response.
We use the DAPO objective, which averages $\min\big(\rho_tA_t,\ \mathrm{clip}(\rho_t,1{-}\varepsilon_{\mathrm{low}},1{+}\varepsilon_{\mathrm{high}})A_t\big)$ over tokens with $(\varepsilon_{\mathrm{low}},\varepsilon_{\mathrm{high}})=(0.2,0.28)$~\citep{yu2026dapo} and a dual clip at $c{=}10$ on negative advantages~\citep{ye2020mastering}.
A token whose ratio leaves this range in the direction favored by its advantage is \emph{clipped}, so its contribution to the policy gradient is zero.
At the first optimizer step after a rollout, every ratio equals one, so the clip can act only from the second step onward.

\begin{definition}[Batches and the square-root rule]\label{def:rule}
A rollout of $P$ prompts with $G$ responses each contains $\Br=PG$ responses, the \emph{rollout batch}.
We split these into $\mu$ mini-batches and take one optimizer step on each, so each step sees the \emph{optimizer batch} $b=\Br/\mu$.
At fixed $\mu$, the \emph{square-root rule} sets $\eta\propto\sqrt{\Br}$, equivalently $\eta\propto\sqrt b$.
As in common practice~\citep{li2026larger}, Adam's other hyperparameters stay fixed.
\end{definition}

\paragraph{Sample efficiency.} Because generation dominates cost, we compare runs at equal numbers of consumed responses $N$.
Let $J_{\Br}(N)$ be the validation accuracy after $N$ responses and $N_\tau(\Br)$ the first $N$ at which it reaches a target $\tau$.
\begin{definition}[Sample cost]\label{def:cost}
The square-root rule is \emph{sample-invariant} if $J_{\Br}(N)$ does not depend on $\Br$.
The \emph{sample-cost exponent} $\alpha=\mathrm{d}\log N_\tau/\mathrm{d}\log\Br$ is $0$ under perfect invariance and $1$ when doubling the rollout batch doubles the samples needed.
\end{definition}

\begin{definition}[Failure point]\label{def:fail}
The \emph{failure point} $N_{\mathrm{fail}}$ of a run is the number of consumed responses at which it first meets a failure criterion fixed in advance (\S\ref{sec:setup}).
A run that does not fail within its training budget is \emph{right-censored}: the budget is a lower bound on its failure point.
\end{definition}

\section{Method: A Two-Clock Model of Batch Scaling}
\label{sec:method}\label{sec:model}

\begin{figure*}[t]
\centering
\includegraphics[width=\textwidth]{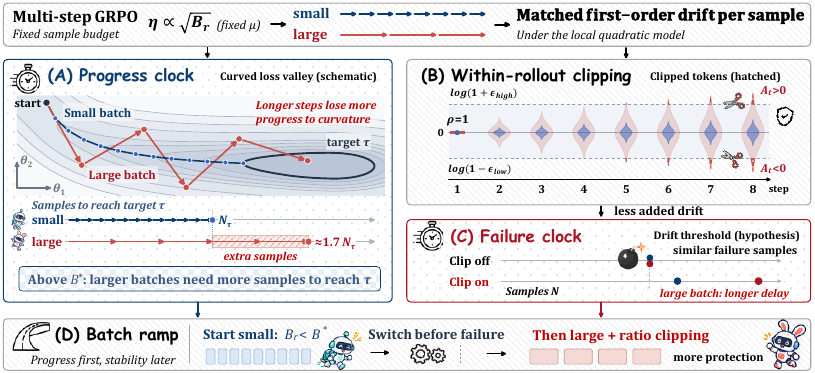}
\caption{Schematic of the two-clock model, with small (blue) and large (red) batches at the same first-order drift per sample.
(A) The large batch covers this drift in fewer, longer steps that lose more progress to curvature, so it needs more samples to reach the target $\tau$.
(B) Token log-ratios spread with each step, about $\sqrt8$ times faster for the large batch, and the clip cuts tokens outside the band.
(C) Unclipped, both fail at similar sample counts, and the clip delays failure, most for the large batch.
(D) A batch ramp thus starts below $B^*$ and switches before failure.}
\label{fig:framework}
\end{figure*}

We now formalize the two-clock picture.
We model one optimizer step on a locally quadratic loss and ask which per-sample quantities the square-root rule preserves.
\textbf{The model first yields a progress law from curvature, then adds a drift-threshold hypothesis for failure, and finally places the ratio clip on the drift accumulated within a rollout.} Figure~\ref{fig:framework} summarizes the picture.
Detailed proofs and extensions are provided in App.~\ref{app:theory}.

\paragraph{Setting.} Let $\hat g=\bar g+\xi$ be the policy-gradient estimate from an optimizer batch of $b$ responses, where $\bar g$ is the population gradient, $\mathbb{E}\,\xi=0$, $\mathrm{Cov}\,\xi=\Sigma/b$, and $\sigma^2=\mathrm{tr}\,\Sigma$.
One optimizer step is $\Delta\theta=-\eta\,u(\hat g)$, and in the main text we use a scalar Adam-type normalizer
\begin{equation}
u(\hat g)=\hat g/\nu,\qquad \nu^2=\mathbb{E}\|\hat g\|^2,
\label{eq:normaliser}
\end{equation}
treating the per-coordinate version in App.~\ref{app:adam}.
We make two assumptions for the progress clock:
\begin{enumerate}[leftmargin=*,itemsep=0pt,label=(A\arabic*)]
  \item \emph{Noise-dominated normalization.} The optimizer batch lies well below the gradient noise scale, so $\nu\approx\sigma/\sqrt b$.
        Our optimizer batches are 60 to 2000 times below the noise scale measured in App.~\ref{app:replay}.
  \item \emph{Local quadratic loss.} Over one step the loss changes by $\langle\bar g,\Delta\theta\rangle+\tfrac12\Delta\theta^{\top}H\Delta\theta$, with $H$ the local Hessian.
        For progress we treat the mini-batches within a rollout as fresh and ignore policy staleness, whose effect on importance ratios and clipping we model separately below.
\end{enumerate}

\paragraph{What the square-root rule fixes.} Under the square-root rule, a larger optimizer batch takes longer steps but fewer of them, and the two effects cancel in the displacement accumulated per sample.
Under (A1) the normalizer shrinks as $1/\sqrt b$ while the rule raises the learning rate as $\sqrt b$, so writing $\eta=\kappa\sqrt b$ the mean step is
\begin{equation}
\mathbb{E}[\Delta\theta]=-\frac{\kappa b}{\sigma}\,\bar g,
\label{eq:meanstep}
\end{equation}
while a fixed sample budget contains only $1/b$ as many steps.
The persistent displacement per consumed response, which we call the \emph{drift per sample},
\begin{equation}
d=\frac{\kappa\,\|\bar g\|}{\sigma},
\label{eq:driftps}
\end{equation}
is therefore independent of the batch (App.~\ref{app:progressproof}).
Here $\kappa$, the learning rate per square root of the optimizer batch, is what the rule holds fixed.

\paragraph{The progress clock.} The square-root rule preserves the drift per sample, but equal movement need not produce equal progress on a curved loss.

\begin{proposition}[Progress per sample]\label{prop:alpha}
Under (A1)--(A2), with $\eta=\kappa\sqrt b$, the expected loss decrease per consumed response is
\begin{equation}
r(b)=r_0-C_{\mathrm{curv}}\,b,\qquad C_{\mathrm{curv}}=\frac{\kappa^2\,\bar g^{\top}H\bar g}{2\sigma^2},
\label{eq:rate}
\end{equation}
where $r_0$ does not depend on $b$.
\end{proposition}
The derivation and the form of $r_0$ are given in App.~\ref{app:progressproof}.
Progress per sample falls linearly with the batch because a larger batch covers the same drift in fewer, longer steps, and those steps lose more progress to curvature.
Reaching a target loss decrease $\Delta L_\tau$ requires $N_\tau=\Delta L_\tau/r(b)$ responses, and to first order in $b$, with $b=\Br/\mu$,
\begin{align}
N_\tau(\Br)&=N_\tau^{\min}\Big(1+\frac{\Br}{B^{*}}\Big),\label{eq:progress}\\
B^{*}&=\mu\Big(\frac{2\sigma\|\bar g\|^2}{\kappa\,\bar g^{\top}H\bar g}-B_{\mathrm{noise}}\Big),\label{eq:bstar}
\end{align}
with $B_{\mathrm{noise}}=\mathrm{tr}(H\Sigma)/\bar g^{\top}H\bar g$.

Unlike the critical batch of \citet{mccandlish2018empirical}, which equals $B_{\mathrm{noise}}$ when the learning rate is tuned at every batch, $B^{*}$ carries a term that falls as $1/\kappa$ because the rule assigns the learning rate.
$B^{*}$ therefore decreases as the learning rate grows and can lie far below the gradient noise scale (App.~\ref{app:progressproof}).
Beyond the small-batch regime we use Eq.~\eqref{eq:progress} as a fitting form.

\paragraph{The failure clock.} For failure we add one hypothesis, which \S\ref{sec:failure} tests.
\begin{enumerate}[leftmargin=*,itemsep=0pt,label=(A\arabic*),start=3]
  \item \emph{Drift threshold.} A run fails when its accumulated drift, $\|\sum_j\mathbb{E}\,\Delta\theta_j\|$, reaches a threshold $D$ over a stretch in which the direction of the mean gradient changes slowly.
\end{enumerate}

\begin{proposition}[Failure law]\label{rem:fail}
Under (A1)--(A3) with a constant learning rate, a run fails after
\begin{equation}
N_{\mathrm{fail}}=C\,\frac{\sqrt b}{\eta},\qquad C=\frac{D\,\sigma}{\|\bar g\|},
\label{eq:law}
\end{equation}
responses.
A noise-driven failure mechanism would instead give $N_{\mathrm{fail}}\propto b/\eta^2$.
\end{proposition}
The derivation is given in App.~\ref{app:failproof}.
Under the square-root rule, $\eta=\kappa\sqrt b$, so Eq.~\eqref{eq:law} reduces to
\begin{equation}
N_{\mathrm{fail}}=\frac{C}{\kappa},
\label{eq:lawrule}
\end{equation}
which is independent of the batch.

\paragraph{Where the clip acts.} Within a rollout, the importance ratio records how far the current policy has moved from the policy that generated the token.
To first order, after $k$ optimizer steps
\begin{equation}
\log\rho_t\approx\sum_{j<k}\langle\Delta\theta_j,s_t\rangle,\qquad s_t=\nabla_\theta\log\pi_\theta(y_t).
\label{eq:logratio}
\end{equation}
Because the square-root rule gives larger batches longer steps, the spread of log-ratios over tokens, the \emph{ratio spread}, grows faster within a rollout, by a factor $\sqrt{\Br^{L}/\Br^{S}}$ at matched $k$ (App.~\ref{app:ratio}).
The clip, which zeroes the contribution of tokens whose ratio crosses its bound in the direction favored by their advantage, therefore acts more strongly on large batches.
Since every ratio resets at the next rollout, the clip limits how much drift each rollout adds rather than undoing accumulated drift, and it does not generate the curvature term behind the progress gap.

\paragraph{Predictions.} The model gives three testable predictions, corresponding to the three phenomena of the introduction.
\begin{enumerate}[leftmargin=*,itemsep=1pt,label=\textbf{P\arabic*.}]
  \item \emph{Progress slows beyond a critical batch.} Sample efficiency degrades with the batch as the curvature term becomes important, more so at higher learning rates, and even where the clip cannot act.
  \item \emph{Failure keeps time in samples.} Without the clip, the square-root rule preserves the failure point in samples, and lowering the learning rate delays failure linearly, as the drift law predicts, not quadratically, as a noise-driven $1/\eta^2$ law would.
  \item \emph{The clip buys stability rather than sustained speed.} Larger batches are clipped more, so the clip should delay their failure most without removing the progress gap, and widening it should give this protection up.
\end{enumerate}

\section{Experimental Setup}
\label{sec:setup}

We train Qwen2.5-Math-1.5B~\citep{yang2024qwen2} with GRPO in verl~\citep{sheng2025hybridflow} on DAPO-Math-17k~\citep{yu2026dapo}, using $G{=}8$, temperature 1, and responses up to 3072 tokens.
We use AdamW~\citep{loshchilov2017decoupled} with a short linear warm-up followed by a constant learning rate.
To isolate batch-size, learning-rate, and clipping effects, we disable KL penalty, entropy bonus, dynamic sampling, and gradient-norm clipping, and test gradient-norm clipping separately.
Each run has a budget of 98{,}304 responses.
We evaluate every 8{,}192 responses on 512 held-out problems, averaging four samples per problem.
Additional details are in App.~\ref{app:details}.

\paragraph{Batches and learning rates.} We use rollout batches of 512, 1536, 4096, and 8192 responses (S, M, L, and XL in Table~\ref{tab:cells}), spanning a $16\times$ range.
Unless stated otherwise, each rollout is split into $\mu{=}8$ optimizer steps.
Learning rates follow the square-root rule from $\eta_S\in\{10^{-6},5{\times}10^{-6},10^{-5}\}$.
At $10^{-6}$, clipping is almost never active, and we call this regime \emph{conservative}.
At the two larger rates, clipping becomes regularly active, and we call these regimes \emph{active}.

\paragraph{Failure criterion.} A run fails at the first point at which its accuracy falls below half of the best accuracy observed so far.
If this never happens, we use the first sustained gradient-norm excursion: the first rollout whose gradient norm exceeds four times its early median and is followed by at least three of the next eight rollouts above the same threshold.
Runs meeting neither condition within the training budget are right-censored.

\paragraph{Controls.} We compare matched runs with clipping removed or widened, and replay frozen checkpoints on identical data with and without clipping.
We repeat key comparisons on DeepScaleR-Preview~\citep{deepscaler2025}, Qwen2.5-1.5B-Instruct~\citep{qwen2.5}, and Qwen2.5-Math-7B~\citep{yang2024qwen2}.
We use two seeds where available.

\begin{table}[t]
\centering\scriptsize
\setlength{\tabcolsep}{3pt}
\begin{tabular}{lrrrrl}
\toprule
cell & $P$ & $\Br$ & $\mu$ & $b$ & $\eta$ \\
\midrule
S & 64 & 512 & 8 & 64 & $\eta_S$ \\
M & 192 & 1536 & 8 & 192 & $\sqrt3\,\eta_S$ \\
L & 512 & 4096 & 8 & 512 & $\sqrt8\,\eta_S$ \\
XL & 1024 & 8192 & 8 & 1024 & $4\,\eta_S$ \\
\midrule
S / L & 64 / 512 & 512 / 4096 & 1 & 512 / 4096 & $\eta_S$ / $\sqrt8\,\eta_S$ \\
S & 64 & 512 & 4 / 16 & 128 / 32 & $\eta_S$ \\
\bottomrule
\end{tabular}
\caption{Configurations.
$P$ is the number of prompts per rollout, $\Br$ the responses per rollout, $\mu$ the optimizer steps per rollout, and $b$ the responses per optimizer step.
We name runs by their rollout batch, S, M, L or XL, with $\mu{=}8$ unless stated otherwise, and the $\mu{=}4$ and $\mu{=}16$ runs are trained without the clip.}
\label{tab:cells}
\end{table}

\section{Experimental Results}
\label{sec:results}

\subsection{Progress slows above a critical batch (P1)}
\label{sec:progress}\label{sec:active}

\begin{figure}[t]
\centering
\includegraphics[width=\linewidth]{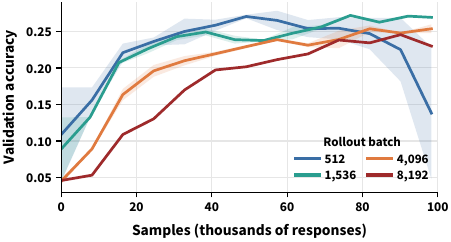}
\caption{Learning curves of four rollout batches under the square-root rule with $\eta_S{=}5{\times}10^{-6}$ and $\mu{=}8$.
Color marks the rollout batch, and bands span two seeds where available.
The 1,536 batch closely tracks the 512 batch, while larger batches require progressively more samples.}
\label{fig:sml}
\end{figure}

P1 predicts that larger batches need more samples once the rollout batch exceeds a critical scale, even when clipping is inactive.
In the active regimes, roughly $10^{-3}$ of tokens are clipped, compared with fewer than $2{\times}10^{-4}$ in the conservative regime.

\paragraph{The gap appears without the clip.} At $\mu{=}1$, where the ratio clip cannot act, Table~\ref{tab:levels} shows that the large batch still needs 20--45\% more samples than the small batch to reach the same accuracy.
The gap also appears at $\mu{=}8$ in the conservative regime and, as Fig.~\ref{fig:clipoff} shows, persists when the clip is removed.
The progress gap therefore does not require the clip.

\paragraph{The penalty appears beyond a critical batch.} Across four rollout batches, increasing the batch from 512 to 1536 costs little, while further increases become progressively less sample-efficient (Fig.~\ref{fig:sml}).
With $\mu{=}8$, this transition occurs at a smaller batch as the learning rate increases.
The sample-cost exponent rises from 0.15 to 0.44 across the three learning-rate regimes, while the critical batch implied by Eq.~\eqref{eq:progress} falls from about 9k to 2k responses (Table~\ref{tab:levels}, Fig.~\ref{fig:levels}, App.~\ref{app:bcrit}).
Below the critical batch the square-root rule is nearly sample-invariant, while above it the sample cost increasingly approaches the step-limited regime.
At the active rates the large batch still ends higher in Table~\ref{tab:levels}, because only the small batch decays or fails late (\S\ref{sec:clip}).

\begin{table}[t]
\centering\scriptsize
\setlength{\tabcolsep}{2.6pt}
\begin{tabular}{llccccc}
\toprule
$\mu$ & $\eta_S$ & $J_S$ & $J_L$ & $\Dauc$ & $N_{.15}$ L/S & $\alpha$ \\
\midrule
1 & $10^{-6}$          & 0.190 & 0.171 & 0.022 & 1.43 & 0.17 \\
1 & $5{\times}10^{-6}$ & 0.272 & 0.255 & 0.020 & 1.23 & 0.13 \\
8 & $10^{-6}$          & 0.251 & 0.240 & 0.016 & 1.40 & 0.15 \\
8 & $5{\times}10^{-6}$ & 0.229/0.048$^{\P}$ & 0.254 & 0.039 & 1.72 & 0.28 \\
8 & $10^{-5}$          & 0.199/0.013$^\ddagger$ & 0.228 & 0.022 & 2.34 & 0.44 \\
\midrule
8$^{\mathrm{D}}$ & $5{\times}10^{-6}$ & 0.305 & 0.347 & 0.029 & 2.04$^{\flat}$ & 0.37 \\
8$^{\mathrm{I}}$ & $2{\times}10^{-6}$ & 0.082 & 0.076 & 0.019 & 6.2$^{\flat}$ & 0.67 \\
8$^{\mathrm{7B}}$ & $2{\times}10^{-6}$ & 0.400$^{\natural}$ & 0.375$^{\natural}$ & 0.057 & 2.00$^{\flat}$ & 0.36 \\
\bottomrule
\end{tabular}
\caption{Small versus large batch.
$J$ is the final accuracy (seed mean), $\Dauc$ the mean distance between the learning curves (seed noise 0.005--0.009), $N_{.15}$ L/S the ratio of the samples needed to reach 15\% accuracy, and $\alpha$ the pooled sample-cost exponent.
Entries marked $^{\P}$ and $^\ddagger$ list both seeds, the second of which fails late.
Rows marked $^{\mathrm{D}}$ use DeepScaleR, $^{\mathrm{I}}$ Qwen2.5-1.5B-Instruct and $^{\mathrm{7B}}$ Qwen2.5-Math-7B with one seed, and values marked $^{\natural}$ are at 49k samples and $^{\flat}$ at a model-specific target.}
\label{tab:levels}
\end{table}

\paragraph{Robustness.} A positive sample-cost gap also appears on DeepScaleR-Preview, Qwen2.5-1.5B-Instruct, and Qwen2.5-Math-7B (Table~\ref{tab:levels}, App.~\ref{app:instruct}).
The direction of the effect holds in every setting.
Above the critical batch, the square-root rule therefore does not preserve sample efficiency.

\subsection{Failure keeps time in samples (P2)}
\label{sec:failure}\label{sec:stability}

\begin{figure*}[t]
\centering
\includegraphics[width=\textwidth]{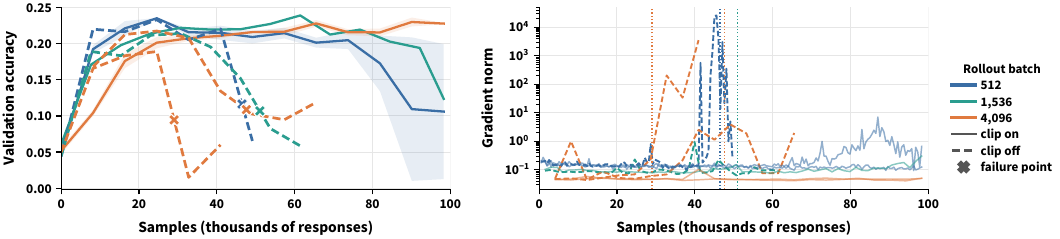}
\caption{Clip on and off for three rollout batches with $\eta_S{=}10^{-5}$ and $\mu{=}8$.
Color marks the rollout batch and line style the clip (solid on, dashed off), with bands over clipped seeds.
Left: accuracy with crosses at failure.
Right: gradient norm with dotted lines at failure.
Unclipped runs fail at similar sample counts despite unequal step counts.}
\label{fig:clipoff}
\vspace{-1em}
\end{figure*}

P2 predicts that without clipping, failure is governed by consumed samples rather than optimizer steps and does not move along the square-root rule.

\paragraph{Failure aligns in samples, not steps.} Removing the clip exposes this clock (Fig.~\ref{fig:clipoff}).
At the gradient-norm onset, the 4,096 batch fails after only 7--10 rollouts, or 56--80 optimizer steps, while the 512 batch lasts 81 rollouts, or 648 steps.
Yet both fail after 30--50k responses.

\paragraph{The square-root rule preserves the failure point.} Across 21 unclipped runs, we vary the optimizer batch by $32\times$ and the learning rate by nearly $6\times$ and fit $N_{\mathrm{fail}}\propto b^{\beta}\eta^{-\gamma}$ (Table~\ref{tab:lawfit}).
Along the square-root rule the failure point does not move with the batch: $\beta-\gamma/2=-0.08$ over the failed runs and $0.11$ with the censored runs included, and both 90\% intervals contain zero, which rules out the batch-blind prediction of $-1/2$.
The fitted $\gamma=0.64$--$0.72$ rules out the random-walk prediction of 2, and of the three candidate forms only the drift form survives a likelihood-ratio test against the free fit ($p\approx0.3$, against $p<10^{-4}$ for the other two).
Fig.~\ref{fig:law} plots each run against the law, and fourteen of the eighteen failing runs lie within a factor of 1.5 of it.

\begin{figure}[t]
\centering
\includegraphics[width=\linewidth]{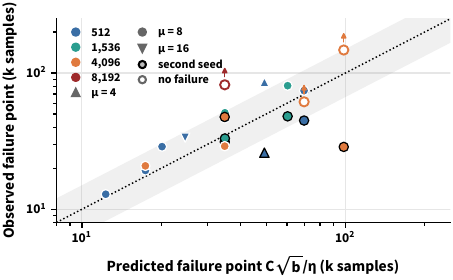}
\caption{Observed failure points of unclipped runs against the failure law, Eq.~\eqref{eq:law}.
Color marks the rollout batch, marker shape $\mu$, and a black edge the second seed, and open markers with arrows are right-censored.
The gray band spans a factor of 1.5 around equality and holds 14 of the 18 failing runs.}
\label{fig:law}
\vspace{-1em}
\end{figure}

\paragraph{The failure pattern favors accumulated drift.} Failure depends on the learning rate far more weakly than the quadratic law of a noise-driven random walk, and at the highest learning rates it coincides with a jump of about four orders of magnitude in the KL divergence between the updated and rollout policies.
The failure point is also not an artifact of disabling gradient-norm clipping: with a norm cap of 1.0, the small batch fails after 36k samples, within the 31--46k range of its seeds.

\paragraph{Robustness.} Failure also keeps time in samples on the 7B model, with a model-specific constant.
On Qwen2.5-Math-7B, the two batches fail $6\times$ apart in rollouts but only $1.3\times$ apart in samples.
Along the square-root rule, larger batches thus fail after the same number of consumed samples despite taking far fewer optimizer steps.

\subsection{The clip buys stability rather than sustained speed (P3)}
\label{sec:clip}\label{sec:mechanism}\label{sec:fix}

\begin{figure}[t]
\centering
\includegraphics[width=\linewidth]{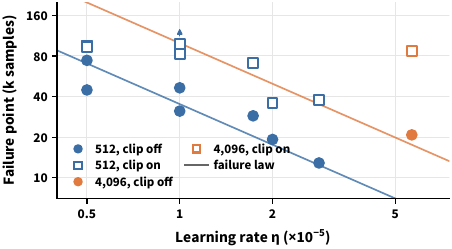}
\caption{Failure point against learning rate for the 512 and 4,096 batches with $\mu{=}8$, with and without the clip.
Color marks the rollout batch, filled circles runs without the clip, open squares runs with it, arrows runs that never failed, and lines the failure law, Eq.~\eqref{eq:law}.
Every clipped run fails later than its unclipped counterpart, and the large batch gains most.}
\label{fig:clipbudget}
\vspace{-1em}
\end{figure}

P3 predicts that the clip, acting on the drift, delays failure, more for the large batch, and that widening it gives up this protection without closing the progress gap.

\paragraph{The clip delays failure, most for the large batch.} With the clip on, the small batch still fails, but later, by a factor of two (Fig.~\ref{fig:clipbudget}).
The large batch, which clips several times more tokens per step, never fails within our training budget at its scaled rates.
Pushed to twice its rate so that it does fail, it lasts four times longer with the clip than without, while the small batch with a similar unclipped failure point gains only a factor of two.
As the model predicts, the stabilizing effect is largest for the batch on which clipping acts most strongly.
A practical side effect: since clipped runs always outlive the unclipped law, Eq.~\eqref{eq:law} is a conservative estimate of when a clipped run can fail.

\vspace{-0.2em}
\paragraph{What the clip removes.} In replays of frozen checkpoints, only about 0.5\% of the large batch's tokens are clipped, yet their gradient has more than half the norm of the full gradient, while the share is much smaller for the small batch (App.~\ref{app:share}, App.~\ref{app:replay}).
Clipping therefore removes a large, batch-dependent part of the update, as the wider ratio spread of the large batch predicts.

\paragraph{Widening the clip trades stability for a transient gain.} Widening the clip reduces the clipped fraction and speeds early learning, but the gain is not sustained: performance later decays and one of two seeds collapses (App.~\ref{app:fix}).
Allowing ratios to move farther therefore gives up much of the large batch's stability without recovering sustained sample efficiency.

Overall, the clip trades a transient speed-up for stability, a trade that favors large batches.

\section{Implications}
\label{sec:implications}\label{sec:guidance}\label{sec:ramp}

\begin{figure}[t]
\centering
\includegraphics[width=\linewidth]{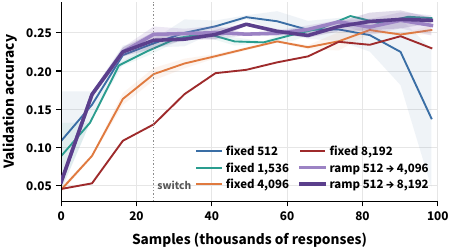}
\caption{Batch ramp with $\eta_S{=}5{\times}10^{-6}$ and $\mu{=}8$.
Each ramp uses batch 512 for 24.6k samples (dotted line), then switches to 4,096 or 8,192 (two seeds each), and thin lines are fixed batches.
Both ramps lie well above the fixed large batches, near the best fixed batch.}
\label{fig:ramp}
\end{figure}

The two clocks give a simple batch-selection rule.
Its one free choice, the switch point, is forgiving over a broad range.

\paragraph{Below the critical batch, use the largest efficient batch.} Since progress changes little below $B^{*}$ while larger batches are no less stable, the model favors the largest batch that remains below the critical regime.
At our learning rates this is 1,536 responses: it matches the small batch's sample efficiency to within seed noise, does not fail, and achieves the best final accuracy among the fixed batches.

\vspace{-0.3em}
\paragraph{Above the critical batch, ramp the batch.} When a larger batch is needed for throughput, train first with a small batch and switch, before its predicted failure point, to a larger and more stable batch.
We reset the optimizer state and repeat the one-rollout warm-up at the switch.
The resulting ramps lie well above the fixed large-batch curves and track the best fixed batch, while generating three quarters of their samples with the large batch (Fig.~\ref{fig:ramp}).
App.~\ref{app:sweep} finds that switches between 25k and 74k samples perform similarly, and the failure law bounds how late the switch can come.

\vspace{-0.3em}
\section{Related Work}
\vspace{-0.3em}
\label{sec:related}

\paragraph{Batch-size scaling and invariance.}
Batch size is a basic scaling variable in optimization, affecting both computational efficiency and training dynamics.
The linear rule for SGD~\citep{smith2018dont}, which also motivates growing the batch during training, and the square-root rule for Adam~\citep{malladi2022sde} keep training trajectories comparable across batch sizes.
\citet{hilton2022batch} make PPO batch-size invariant by decoupling the clip's proximal policy from the behavior policy, given a sufficiently small optimizer batch, while \citet{li2026larger} show that square-root scaling keeps GRPO learning curves aligned with one optimizer step per rollout over a bounded range of batch sizes.
In contrast, we study the multi-step regime, where this batch-size invariance no longer holds and drift preservation does not imply progress preservation.

\paragraph{Critical batch.}
Beyond a critical batch, larger batches stop saving optimizer steps and start costing samples.
With the learning rate re-tuned at every batch, the critical batch of \citet{mccandlish2018empirical} equals the gradient noise scale and does not depend on the learning rate.
Under square-root scaling, by contrast, the learning rate is coupled to batch size, leading to a different critical batch that depends on the learning-rate scale.
In LLM RL, larger batches have been found to reach higher final accuracy~\citep{depthbreadth2025,scalerl} while learning more slowly~\citep{scalerl}.
Our two-clock view connects these observations by separating the batch scale governing progress from the sample scale governing failure.

\paragraph{Training collapse.}
Batch scaling also affects training stability.
\citet{stepsize2025} derive a step-size threshold above which RLVR collapses, while \citet{staleness2026} find that collapse time in asynchronous GRPO scales roughly as $1/\eta$, the fixed-batch case of Eq.~\eqref{eq:law}, and attribute it to directional drift.
We add the batch dimension: along the square-root rule, failure is approximately invariant in consumed samples.

\paragraph{Ratio clipping.}
Training stability in GRPO is closely tied to ratio clipping.
Trust-region methods replace or soften the hard clip~\citep{qi2026rethinking,pspo2026,yang2026clipping,gacpo2026}, while \citet{mugrpo2026} relax it to enable stale-rollout reuse and \citet{scalerl} show that the clip threshold changes RL scaling behavior.
In contrast to work that treats clipping as a bottleneck~\citep{qi2026rethinking,yang2026clipping}, we find that it mainly buys stability, at a transient cost in speed.

\section{Conclusion}
\label{sec:conclusion}

Scaling rules are often interpreted as promises of invariance.
For multi-step GRPO, the square-root rule instead preserves first-order drift per sample, not progress per sample, yielding two clocks: progress slows beyond a learning-rate-dependent critical batch, while a drift-threshold hypothesis predicts an approximately sample-invariant failure point.
Our experiments support these predictions and show that ratio clipping buys stability, especially for larger batches, while the progress gap persists without it.
Together, the two clocks suggest a simple strategy for multi-step RLVR: train below the critical batch for efficient progress, then switch to a larger batch before failure when throughput demands it.

\section*{Limitations}
We broaden our evaluation across multiple checkpoints, two datasets, and a 7B model, but compute and resource constraints still limit our coverage to one model family and one task domain.
Our theory is intentionally local, and direct measurement of the curvature term is computationally expensive at this scale, so we evaluate it through its testable predictions.
Failure points also exhibit intrinsic variation across seeds.
Finally, we measure sample efficiency in responses, while token, compute, and wall-clock efficiency additionally depend on generation lengths, hardware, and system implementation.

\section*{Ethical Considerations}
This work analyzes the training dynamics of RLVR on public math datasets with open-weight models and releases no new model or dataset.
Its main societal relevance is environmental: our batch schedule aims to reduce the number of generated samples a large batch needs to reach a given accuracy, although our own study required many training runs (App.~\ref{app:details} reports the hardware).
We see no additional risks beyond those of LLM post-training in general.

\bibliography{references}

\newpage
\appendix

\section{Theory and Derivations}
\label{app:theory}\label{app:scaling}
This appendix gives the derivations behind \S\ref{sec:method} in full: the progress law (\ref{app:progressproof}), the failure law (\ref{app:failproof}), the ratio spread (\ref{app:ratio}), the share of the clipped tokens (\ref{app:share}), Adam's per-coordinate normalization and larger batches (\ref{app:adam}), and the critical batches our measurements imply (\ref{app:bcrit}).
Throughout, $b=\Br/\mu$ is the optimizer batch, $\hat g=\bar g+\xi$ the gradient estimate from $b$ responses with $\mathbb{E}\,\xi=0$ and $\mathrm{Cov}\,\xi=\Sigma/b$, $\sigma^2=\mathrm{tr}\,\Sigma$, and $H$ the local Hessian.

\subsection{Progress per sample}
\label{app:progressproof}
\paragraph{The normalized step.}
The scalar normalizer is $\nu^2=\mathbb{E}\|\hat g\|^2$.
Since $\xi$ has mean zero,
\begin{equation}
\nu^2=\|\bar g\|^2+\frac{\sigma^2}{b}=\frac{\sigma^2}{b}\Big(1+\frac{b}{\Bsimple}\Big),
\label{eq:nu}
\end{equation}
with $\Bsimple=\sigma^2/\|\bar g\|^2$.
Under (A1), $b\ll\Bsimple$, so $\nu\approx\sigma/\sqrt b$ to leading order in $b/\Bsimple$, and the step $\Delta\theta=-\eta\,\hat g/\nu$ becomes
\begin{equation}
\Delta\theta\approx-\frac{\eta\sqrt b}{\sigma}\,(\bar g+\xi).
\label{eq:stepapp}
\end{equation}

\paragraph{Expected loss change.}
Under (A2) the loss change over one step is
\begin{equation}
\Delta L=\langle\bar g,\Delta\theta\rangle+\tfrac12\,\Delta\theta^{\top}H\Delta\theta.
\label{eq:quad}
\end{equation}
Substituting Eq.~\eqref{eq:stepapp} and taking the expectation over $\xi$ term by term:
the linear term gives
\begin{equation}
\mathbb{E}\,\langle\bar g,\Delta\theta\rangle=-\frac{\eta\sqrt b}{\sigma}\,\|\bar g\|^2,
\label{eq:lin}
\end{equation}
because $\mathbb{E}\,\xi=0$.
The quadratic term gives
\begin{align}
\mathbb{E}\,\Delta\theta^{\top}H\Delta\theta&=\frac{\eta^2 b}{\sigma^2}\Big(\bar g^{\top}H\bar g+\mathbb{E}\,\xi^{\top}H\xi\Big)\notag\\
&=\frac{\eta^2}{\sigma^2}\Big(b\,\bar g^{\top}H\bar g+\mathrm{tr}(H\Sigma)\Big),
\label{eq:quadexp}
\end{align}
because the cross terms vanish and $\mathbb{E}\,\xi^{\top}H\xi=\mathrm{tr}(H\,\mathrm{Cov}\,\xi)=\mathrm{tr}(H\Sigma)/b$.
Together,
\begin{equation}
\mathbb{E}\,\Delta L=-\frac{\eta\sqrt b}{\sigma}\,\|\bar g\|^2+\frac{\eta^2}{2\sigma^2}\Big(b\,\bar g^{\top}H\bar g+\mathrm{tr}(H\Sigma)\Big).
\label{eq:expand}
\end{equation}

\paragraph{Proof of Proposition~\ref{prop:alpha}.}
One step consumes $b$ responses, so the expected loss decrease per consumed response is $r=-\mathbb{E}\,\Delta L/b$.
Dividing Eq.~\eqref{eq:expand} by $b$ gives
\begin{equation}
r=\frac{\eta}{\sqrt b}\,\frac{\|\bar g\|^2}{\sigma}-\frac{\eta^2}{b}\,\frac{b\,\bar g^{\top}H\bar g+\mathrm{tr}(H\Sigma)}{2\sigma^2}.
\label{eq:rateraw}
\end{equation}
With $\eta=\kappa\sqrt b$ we have $\eta/\sqrt b=\kappa$ and $\eta^2/b=\kappa^2$, which yields Eq.~\eqref{eq:rate} with
\begin{equation}
r_0=\frac{\kappa\,\|\bar g\|^2}{\sigma}-\frac{\kappa^2\,\mathrm{tr}(H\Sigma)}{2\sigma^2}.
\label{eq:r0}
\end{equation}
The first two terms do not depend on $b$, and only the curvature of the mean step, $b\,\bar g^{\top}H\bar g$, grows with the batch. $\square$

\paragraph{From the rate to the critical batch.}
Write $c=C_{\mathrm{curv}}$ for the slope in Eq.~\eqref{eq:rate}.
If $r$ is roughly constant over the stretch of training in which the loss falls by $\Delta L_\tau$, the samples needed are
\begin{equation}
N_\tau=\frac{\Delta L_\tau}{r_0-c\,b}=\frac{\Delta L_\tau}{r_0}\Big(1+\frac{c\,b}{r_0}+O\big((c\,b/r_0)^2\big)\Big).
\label{eq:ntauexp}
\end{equation}
Keeping the first-order term and writing $b=\Br/\mu$ gives Eq.~\eqref{eq:progress} with $N_\tau^{\min}=\Delta L_\tau/r_0$ and
\begin{equation}
B^{*}=\mu\,\frac{r_0}{c}=\mu\Big(\frac{2\sigma\|\bar g\|^2}{\kappa\,\bar g^{\top}H\bar g}-\frac{\mathrm{tr}(H\Sigma)}{\bar g^{\top}H\bar g}\Big),
\label{eq:bstarapp}
\end{equation}
which is Eq.~\eqref{eq:bstar}.
The expansion requires $r_0>0$, that is $\kappa<2\sigma\|\bar g\|^2/\mathrm{tr}(H\Sigma)$, and $c\,b\ll r_0$, that is $\Br\ll B^{*}$.
The local sample-cost exponent of Definition~\ref{def:cost} follows by differentiating the logarithm of Eq.~\eqref{eq:progress}:
\begin{equation}
\alpha=\frac{\mathrm{d}\log N_\tau}{\mathrm{d}\log\Br}=\frac{\Br/B^{*}}{1+\Br/B^{*}}=\frac{\Br}{\Br+B^{*}}.
\label{eq:alphalocal}
\end{equation}

\paragraph{Contrast with a learning rate tuned at every batch.}
For SGD with the learning rate $\epsilon$ re-tuned at every batch~\citep{mccandlish2018empirical}, the same expansion gives $\mathbb{E}\,\Delta L=-\epsilon\|\bar g\|^2+\tfrac{\epsilon^2}{2}\big(\bar g^{\top}H\bar g+\mathrm{tr}(H\Sigma)/b\big)$.
Minimizing over $\epsilon$ gives the best decrease per step,
\begin{equation}
-\mathbb{E}\,\Delta L^{*}=\frac{\|\bar g\|^4}{2\,\bar g^{\top}H\bar g}\cdot\frac{1}{1+B_{\mathrm{noise}}/b},
\label{eq:tuned}
\end{equation}
with $B_{\mathrm{noise}}=\mathrm{tr}(H\Sigma)/\bar g^{\top}H\bar g$, the noise scale weighted by curvature.
The samples per unit of progress are then proportional to $b\,(1+B_{\mathrm{noise}}/b)=b+B_{\mathrm{noise}}$, so the critical batch equals $B_{\mathrm{noise}}$ and does not depend on the learning rate.
Under the square-root rule the learning rate is not tuned but assigned, and Eq.~\eqref{eq:bstarapp} equals $\mu$ times $2\sigma\|\bar g\|^2/(\kappa\,\bar g^{\top}H\bar g)-B_{\mathrm{noise}}$: the first term falls as $1/\kappa$, which is why $B^{*}$ moves with the learning rate.

\paragraph{Noise per sample.}
The noise part of one step in Eq.~\eqref{eq:stepapp} is $-(\eta\sqrt b/\sigma)\,\xi$, with covariance
\begin{equation}
\frac{\eta^2 b}{\sigma^2}\cdot\frac{\Sigma}{b}=\frac{\eta^2}{\sigma^2}\,\Sigma=\kappa^2 b\,\frac{\Sigma}{\sigma^2}.
\label{eq:noisecov}
\end{equation}
Fresh mini-batches give independent noise, so over the $N/b$ steps of a run consuming $N$ responses the covariances add to $(N/b)\,\kappa^2 b\,\Sigma/\sigma^2=N\kappa^2\Sigma/\sigma^2$.
The noise covariance accumulated per sample, $\kappa^2\Sigma/\sigma^2$, is therefore independent of the batch, as is the drift per sample.

\subsection{Failure law}
\label{app:failproof}
\paragraph{Proof of Proposition~\ref{rem:fail}.}
By Eq.~\eqref{eq:stepapp}, the mean of step $j$ is $\mathbb{E}\,\Delta\theta_j=-(\eta\sqrt b/\sigma)\,\bar g_j$, where $\bar g_j$ is the mean gradient at step $j$.
Over the stretch assumed in (A3) the direction of $\bar g_j$ changes slowly, so the means add almost collinearly, and after $k$ steps
\begin{equation}
\Big\|\sum_{j<k}\mathbb{E}\,\Delta\theta_j\Big\|\approx k\,\frac{\eta\sqrt b}{\sigma}\,\|\bar g\|.
\label{eq:driftk}
\end{equation}
The run fails when this reaches $D$, after
\begin{equation}
k_{\mathrm{fail}}=\frac{D\,\sigma}{\eta\sqrt b\,\|\bar g\|}
\label{eq:kfail}
\end{equation}
steps.
Each step consumes $b$ responses, so
\begin{equation}
N_{\mathrm{fail}}=k_{\mathrm{fail}}\,b=\frac{D\,\sigma}{\|\bar g\|}\cdot\frac{\sqrt b}{\eta}=C\,\frac{\sqrt b}{\eta}.
\label{eq:nfailapp}
\end{equation}
Under the square-root rule $\eta=\kappa\sqrt b$, so $N_{\mathrm{fail}}=C/\kappa$. $\square$

\paragraph{Alternative mechanisms.}
If failure were instead driven by the noise part of the steps, the relevant displacement would be the sum of independent noise terms.
By Eq.~\eqref{eq:noisecov} each has total variance $\mathrm{tr}(\eta^2\Sigma/\sigma^2)=\eta^2$, so after $k$ steps the root-mean-square displacement is $\sqrt k\,\eta$.
It reaches $D$ after $k\approx D^2/\eta^2$ steps, or
\begin{equation}
N_{\mathrm{fail}}^{\mathrm{noise}}\approx\frac{b\,D^2}{\eta^2}
\label{eq:nfailnoise}
\end{equation}
responses.
A batch-blind mechanism, in which failure depends only on the learning rate, gives $N_{\mathrm{fail}}\propto1/\eta$.
Writing each as $N_{\mathrm{fail}}\propto b^{\beta}\eta^{-\gamma}$, the drift law has $(\beta,\gamma)=(\tfrac12,1)$, the noise law $(1,2)$ and the batch-blind law $(0,1)$.
Along the square-root rule $\eta\propto\sqrt b$, so $N_{\mathrm{fail}}\propto b^{\beta-\gamma/2}$: the drift and noise laws both predict $\beta-\gamma/2=0$, while the batch-blind law predicts $-\tfrac12$.
Invariance along the rule therefore separates the drift law from the batch-blind law, and the learning-rate exponent $\gamma$ separates it from the noise law, which is how the fits in Table~\ref{tab:lawfit} tell the three apart.

\begin{corollary}[Two clocks]\label{prop:clocks}
Under (A1)--(A3) and the square-root rule at fixed $\mu$, the failure point of Eq.~\eqref{eq:law} does not depend on $\Br$, while the samples to a target, $N_\tau=\Delta L_\tau/r(\Br/\mu)$, increase with $\Br$.
\end{corollary}
\begin{proof}
Eq.~\eqref{eq:nfailapp} gives $N_{\mathrm{fail}}=C/\kappa$, and $\kappa$ is fixed by the rule.
By Eq.~\eqref{eq:rate}, $r$ decreases linearly in $b=\Br/\mu$ whenever $\bar g^{\top}H\bar g>0$, so $N_\tau=\Delta L_\tau/r$ increases with $\Br$.
\end{proof}

\paragraph{On (A3).} The threshold assumption is the model's one empirical hypothesis, and the failure-law fit of \S\ref{sec:failure} tests the exponents it predicts.
The threshold concerns movement along the slowly varying mean-gradient direction, where drift accumulates linearly in the samples consumed, whereas noise, spread over many directions, contributes only a random-walk component.
It predicts $\gamma=1$ and $\beta=\tfrac12$, the noise-driven alternative $\gamma=2$ and $\beta=1$, and a batch-blind drift $\gamma=1$ and $\beta=0$.
The fit excludes the second, since Table~\ref{tab:lawfit} gives $\gamma\approx0.7$, and pins the combination $\beta-\gamma/2$ that the scaling rule probes near zero, which (A3) predicts and the third alternative does not.

\paragraph{Clip-active steps.} For the same $N$ samples a run takes $N/b$ optimizer steps, of which the fraction $(\mu-1)/\mu$ with $k\ge1$ can clip: 1344 for S against 168 for L at 98k.
If failure required accumulating some number of clip-active steps, the small batch would reach it after eight times fewer samples.
The unclipped runs rule this out: at the gradient-norm onset, L fails after 7 rollouts, or 56 optimizer steps (29k samples), and S after 81 rollouts, or 648 optimizer steps (41k), a $1.4\times$ difference in samples against $12\times$ in steps.
The underlying instability counts samples, as (A3) assumes, and the different clipped failure points must come from the clip's protection differing between the batches (Proposition~\ref{prop:share}).

\subsection{Ratio spread}
\label{app:ratio}
\paragraph{The log-ratio as accumulated movement.}
Let $\theta_0$ be the policy that generated a rollout and $\theta_k$ the policy after $k$ optimizer steps on it.
To first order in the steps, $\log\pi_{\theta_k}(y_t)-\log\pi_{\theta_0}(y_t)\approx\langle\theta_k-\theta_0,s_t\rangle$, with the score $s_t=\nabla_\theta\log\pi_{\theta_0}(y_t)$, so
\begin{equation}
\log\rho_t(k)\approx\sum_{j<k}\langle\Delta\theta_j,s_t\rangle.
\label{eq:logratioapp}
\end{equation}
Define the score spread $\sigma_s=\sup_{\|v\|=1}\mathrm{std}_t\,\langle v,s_t\rangle$, the largest standard deviation over tokens of the score projected onto a unit direction.

\begin{lemma}[Ratio spread]\label{lem:drift}
To first order, the spread of $\log\rho$ over tokens after $k$ steps is at most $k\,\eta\,\sigma_s$, and this bound is reached when the steps are aligned.
When the steps are uncorrelated, the spread is of order $\sqrt k\,\eta\,\sigma_s$ in root-mean-square.
In either case it is proportional to $\eta$, so under the square-root rule at matched $k$ the large batch's spread is $\sqrt{\Br^{L}/\Br^{S}}$ times the small batch's.
\end{lemma}
\begin{proof}
Under (A1) the normalized update $u=\sqrt b\,(\bar g+\xi)/\sigma$ has $\mathbb{E}\|u\|^2=1+b/\Bsimple\approx1$, so each step has norm close to $\eta$.
Each term of Eq.~\eqref{eq:logratioapp} therefore has spread over tokens $\mathrm{std}_t\langle\Delta\theta_j,s_t\rangle\le\|\Delta\theta_j\|\,\sigma_s\approx\eta\,\sigma_s$.
The standard deviation of a sum is at most the sum of the standard deviations, which gives $k\,\eta\,\sigma_s$, with equality when the $k$ terms are perfectly correlated across tokens.
If the terms are uncorrelated, their variances add, giving $\sqrt k\,\eta\,\sigma_s$ in root-mean-square.
For the batch dependence, note that under (A1) the noise part of $u$, $\sqrt b\,\xi/\sigma$, has covariance $\Sigma/\sigma^2$, independent of $b$, so the step directions have the same statistics at every batch and the spread scales exactly with $\eta$.
At fixed $\mu$ the square-root rule sets $\eta_L/\eta_S=\sqrt{\Br^{L}/\Br^{S}}$.
\end{proof}

\paragraph{Measured ratio spread.} The replay measures the spread of Lemma~\ref{lem:drift} at matched $k$ on the final active checkpoints used in \S\ref{sec:clip}: the 99th percentile of $|\log\rho|$ is $1.05$, $1.9$, $3.1$, $2.7$, $2.7$, $2.8$, $2.8$ times larger for L than for S at $k=1,\dots,7$.
From $k{=}3$ on the ratio is the predicted $\sqrt8\approx2.8$.
At $k\le2$ both batches sit near the $0.09$--$0.10$ floor that the conservative runs show at every $k$.
Fig.~\ref{fig:perk} shows that the clipped \emph{fraction} rises with $k$ for both batches, and its ratio between the batches grows from 1.1--1.6 at $k{=}1$ to 2--4 from $k{=}3$ on, around the spread ratio.

\subsection{Share of the clipped tokens}
\label{app:share}
Our replay accumulates the squared norms of the clipped and unclipped gradients and their inner product over parameter shards, without storing either gradient.
The following identity turns these three numbers into the share of the update carried by the clipped tokens.

\begin{proposition}[Share of the clipped tokens]\label{prop:share}
Let $g$ and $g^{\mathrm{clip}}$ be the gradients of the unclipped and clipped objectives on the same mini-batch, $q=\|g^{\mathrm{clip}}\|/\|g\|$ and $\dclip=1-\cos(g^{\mathrm{clip}},g)$.
The clipped tokens contribute $g_C=g-g^{\mathrm{clip}}$, and
\begin{align}
\frac{\|g_C\|}{\|g\|}&=\sqrt{1+q^2-2q(1-\dclip)}\notag\\
&\ge\sqrt{\dclip(2-\dclip)}.
\label{eq:share}
\end{align}
\end{proposition}
\begin{proof}
The clipped objective's gradient omits exactly the contributions of the clipped tokens, so $g=g^{\mathrm{clip}}+g_C$.
By the law of cosines,
\begin{align}
\|g_C\|^2&=\|g\|^2+\|g^{\mathrm{clip}}\|^2\notag\\
&\quad-2\,\|g\|\,\|g^{\mathrm{clip}}\|\cos(g^{\mathrm{clip}},g),
\label{eq:cosines}
\end{align}
and dividing by $\|g\|^2$ gives the equality.
For the bound, minimize $f(q)=1+q^2-2q(1-\dclip)$ over $q\ge0$.
If $\dclip\le1$, the minimum is at $q=1-\dclip$ with value $1-(1-\dclip)^2=\dclip(2-\dclip)$.
If $\dclip>1$, the minimum is at $q=0$ with value $1\ge\dclip(2-\dclip)$.
The bound uses the angle alone and holds whatever the norm ratio.
\end{proof}

\paragraph{Measured share.} Both $q$ and $\dclip$ are measured in every replay, so Eq.~\eqref{eq:share} is exact per step.
In the active regime the gradient of the tokens the clip removes, $0.5\%$ of the mini-batch for L, has $50$--$70\%$ of the norm of the full gradient at every $k\ge2$ (mean $0.50$--$0.55$ over $k$, and up to $0.9$--$1.7$ at the earliest checkpoint, step~2, where the clipped part is larger than the full gradient).
For S, the $0.05\%$ of clipped tokens give $10$--$20\%$ (mean $0.13$--$0.15$), rising to $0.42$ with a step at $1.0$ in the pre-collapse checkpoint of seed~1.
In the conservative regime the share is $\le0.14$ at every step, the level set by the few clipped tokens' sampling noise.
Clipping in the active regime therefore does not trim a tail: it removes a dominant component of the update, and a component that is dominant in the large batch at every step of the rollout.

\subsection{Per-coordinate normalization and larger batches}
\label{app:adam}
\paragraph{Per-coordinate normalization.}
Adam divides coordinate $i$ by $\sqrt{v_i}$, with $v_i=\mathbb{E}\,\hat g_i^2=\bar g_i^2+\Sigma_{ii}/b$.
If (A1) holds coordinate-wise, $b\ll\Sigma_{ii}/\bar g_i^2$, then $v_i\approx\Sigma_{ii}/b$ and
\begin{equation}
\Delta\theta\approx-\eta\sqrt b\,P(\bar g+\xi),\qquad P=\mathrm{diag}(\Sigma_{ii})^{-1/2}.
\label{eq:adamstep}
\end{equation}
Repeating the calculation of \ref{app:progressproof} with this step gives
\begin{align}
r(b)&=\kappa\,\bar g^{\top}P\bar g-\frac{\kappa^2}{2}\,\mathrm{tr}(PHP\Sigma)\notag\\
&\quad-b\,\frac{\kappa^2}{2}\,\bar g^{\top}PHP\bar g,
\label{eq:rateadam}
\end{align}
which has the structure of Eq.~\eqref{eq:rate}, with $\bar g^{\top}P\bar g$ in place of $\|\bar g\|^2/\sigma$, $\mathrm{tr}(PHP\Sigma)$ in place of $\mathrm{tr}(H\Sigma)/\sigma^2$ and $\bar g^{\top}PHP\bar g$ in place of $\bar g^{\top}H\bar g/\sigma^2$.
The drift per sample is $\kappa\|P\bar g\|$, again independent of $b$, and the failure law keeps the form of Eq.~\eqref{eq:law} with $C=D/\|P\bar g\|$.
Momentum averages consecutive steps, which changes the constants but not the dependence on $b$, as long as the gradient changes slowly over the averaging window.

\begin{remark}[Outside (A1)]\label{rem:outside}
By Eq.~\eqref{eq:nu}, the first-order progress per sample is
\begin{equation}
\frac{\eta\,\|\bar g\|^2}{b\,\nu}=\frac{\kappa\,\|\bar g\|^2}{\sigma}\Big(1+\frac{b}{\Bsimple}\Big)^{-1/2}.
\label{eq:outside}
\end{equation}
For $b\ll\Bsimple$ this is the batch-independent first term of Eq.~\eqref{eq:rate}.
When $b$ approaches $\Bsimple$ it falls as $1/\sqrt b$, so even without curvature the sample-cost exponent rises toward $\alpha=\tfrac12$, and once the curvature term dominates, toward $\alpha=1$.
Our optimizer batches are 60 to 2000 times below the measured $\Bsimple$, so this regime does not explain our gaps.
\end{remark}

\subsection{Implied critical batch}
\label{app:bcrit}
\begin{remark}[Implied critical batch]
\label{rem:bcrit}
For a family with a small batch $B_S$ and a large batch $8B_S$, let $R=N_\tau(8B_S)/N_\tau(B_S)=8^{\alpha}$ be the measured ratio of samples to the same target.
Under Eq.~\eqref{eq:progress},
\begin{equation}
R=\frac{1+8B_S/B^{*}}{1+B_S/B^{*}},
\label{eq:twopoint}
\end{equation}
and solving for $B^{*}$ gives
\begin{equation}
B^{*}=\frac{(8-R)\,B_S}{R-1},
\label{eq:bstarfromR}
\end{equation}
with $B_S=512$ responses per rollout.
At $\mu{=}8$, our exponents give $B^{*}\approx9.3$k, $4.0$k (a direct fit over four batches gives $2.2$--$4.2$k) and $1.9$k at $\eta_S{=}10^{-6}$, $5{\times}10^{-6}$ and $10^{-5}$.
The critical batch thus falls as the learning rate grows, in the direction Eq.~\eqref{eq:bstar} predicts.
All three values lie well below $\Bsimple$: as Eq.~\eqref{eq:bstar} shows, under the square-root rule the critical batch carries a term that scales as $1/\kappa$ and is not set by the noise scale alone.
\end{remark}

\section{Experimental Details}
\label{app:details}\label{app:runs}

\paragraph{Design choices.} We disable gradient-norm clipping because a fixed norm cap acts differently across batch sizes and therefore interferes with the square-root rule, and we test it separately in \S\ref{sec:failure}.
In the sustained-excursion rule, when fewer than eight rollouts remain in a run, two of the remaining rollouts above the threshold suffice.
The failure criterion checks accuracy first because some runs collapse quietly, with accuracy falling while the gradient norm stays at its baseline, and App.~\ref{app:onset} lists these runs.

\begin{table}[H]
\centering\scriptsize
\setlength{\tabcolsep}{3pt}
\begin{tabular}{lp{4.3cm}}
\toprule
prompts per rollout $P$ & 64 (S), 192 (M), 512 (L), 1024 (XL) \\
responses per prompt $G$ & 8 \\
optimizer steps per rollout $\mu$ & 8 (main), 1, 4 and 16 (controls), mini-batch of $P/\mu$ prompts \\
learning rate & $\eta_S\in\{10^{-6},5{\times}10^{-6},10^{-5}\}$ and the clip-off sweeps of App.~\ref{app:onset},
  other batches by the square-root rule \\
schedule & linear warm-up over 8192 responses, then constant \\
optimizer & AdamW, $\beta=(0.9,0.999)$, weight decay 0.1 \\
gradient-norm clip & disabled ($10^{6}$), 1.0 in one control \\
ratio clip & $\varepsilon_{\mathrm{low}}=0.2$, $\varepsilon_{\mathrm{high}}=0.28$, dual clip $c=10$ \\
KL / entropy terms & none \\
loss aggregation & token mean over the mini-batch \\
advantage & GRPO group normalization, $\epsilon=10^{-6}$ \\
sampling & temperature 1.0, top-$p$ 1.0, top-$k$ off (training and evaluation) \\
lengths & prompt $\le1024$, response $\le3072$ tokens, truncated responses kept without penalty \\
micro-batch & 4 sequences, fixed \\
data & DAPO-Math-17k (main), DeepScaleR-Preview (second dataset), less than one epoch \\
validation & 512 held-out prompts, mean@4, every 8192 responses \\
precision & bf16 training and rollouts, fp32 master weights and gradient reduction \\
parallelism & 4-GPU FSDP with parameter and optimizer offload, one vLLM engine per GPU \\
software & verl~\citep{sheng2025hybridflow} 0.9.1, vLLM~\citep{kwon2023efficient} 0.19 \\
\bottomrule
\end{tabular}
\caption{Hyperparameters shared by all runs.}
\label{tab:hparams}
\end{table}

\paragraph{Budgets.} Unless stated otherwise, runs consume 98{,}304 responses, and selected failure probes use longer budgets.

\paragraph{Choosing the learning rates.}
\label{app:ladder}
We first ran 24-rollout probes of the small batch over $\eta\in\{10^{-6},5{\times}10^{-6},10^{-5},2{\times}10^{-5}\}$.
Clipping was essentially inactive at $10^{-6}$, became regularly active at $5{\times}10^{-6}$ and $10^{-5}$, and was stronger at $2{\times}10^{-5}$.
We therefore use $10^{-6}$ as the conservative rate and the next two as the active rates.
Separate ladders are used for $\mu{=}1$, Qwen2.5-1.5B-Instruct, and Qwen2.5-Math-7B, and App.~\ref{app:instruct} describes the Instruct ladder.

\paragraph{Exact ratios at the first step.}
\label{app:noise}
Dynamic token packing makes the rollout and the recomputed log-probabilities use differently shaped micro-batches, and under bf16 this produces spurious deviations from $\rho=1$ at $k{=}0$, with the 99th percentile of $|\log\rho|$ reaching 0.03--0.12.
Fixed micro-batches of four sequences in identical order make the ratio exactly one.
All runs use this setting except the seed-0 small-batch runs in the conservative regime, where this noise stays inside the clip range: no token is clipped at $k{=}0$ in that regime, as App.~\ref{app:progress} confirms.

\paragraph{Hardware and compute.}
The 1.5B experiments used four NVIDIA RTX A6000 GPUs per run, while the 7B experiments used four H100 GPUs.
In total, the experiments used roughly 2{,}000 GPU-hours, including preliminary runs and the replay analyses.

\section{Additional Results for the Progress Clock}
\label{app:progress}\label{app:gap}\label{sec:conservative}
\begin{figure*}[t]
\centering
\includegraphics[width=\textwidth]{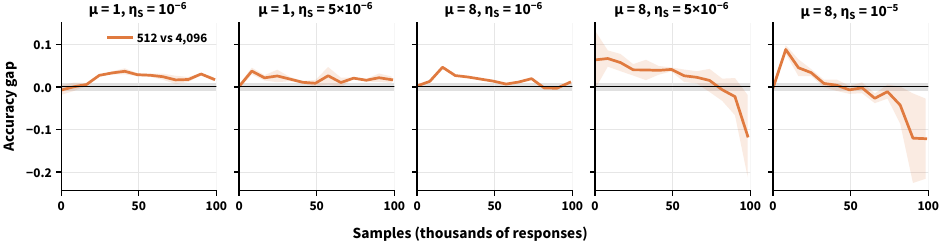}
\caption{Accuracy gap between the 512 and 4,096 batches, $J_{512}(N)-J_{4096}(N)$, for all five pairs.
The two left panels have $\mu{=}1$, where the clip cannot act, and the three right panels $\mu{=}8$.
Bands cover all seed pairings, and the gray band marks $\pm1$ accuracy point, the reference band that prior one-step scaling results use to judge approximate invariance~\citep{li2026larger}.
At $\mu{=}8$ the conservative gap is an early lead of 2--4 points that decays into the band, the $5{\times}10^{-6}$ gap reaches 8 points and stays above the band until the small batch decays late, and the $10^{-5}$ gap peaks at 9 points early and turns negative when the small batch fails.}
\label{fig:gapall}
\end{figure*}
\begin{figure*}[t]
\centering
\includegraphics[width=\textwidth]{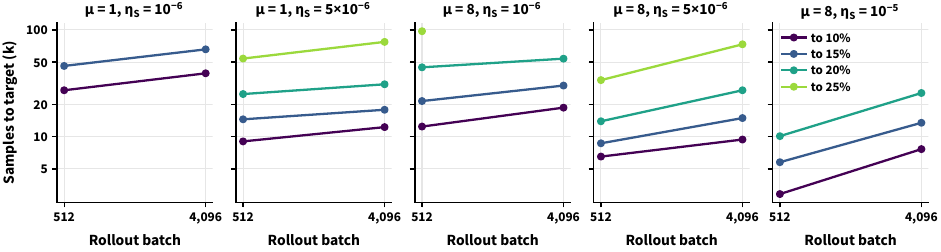}
\caption{Samples needed to first reach each accuracy target against the rollout batch, on log--log axes, for every pair.
The slope of each line is the sample-cost exponent $\alpha$ of Table~\ref{tab:levels}, a missing point means that the run never reaches the target within 98k samples, and the family with $\eta_S{=}5{\times}10^{-6}$, $\mu{=}8$ and all four batches is in Fig.~\ref{fig:clocks} (left).
The samples needed rise with the batch at every target.}
\label{fig:n2t}
\end{figure*}
This appendix gives the evidence behind \S\ref{sec:progress}.
In Table~\ref{tab:levels}, $\alpha$ is the slope of $\log N_\tau$ against $\log\Br$, pooled over the accuracy targets each pair reaches and, for the $5{\times}10^{-6}$ family, over the middle batch as well.

\paragraph{Learning-rate dependence.} At $\mu{=}8$ the small batch's lead over the large batch grows from the conservative to the active rates, as Fig.~\ref{fig:levels} shows, the dependence that the sample-cost exponents of Table~\ref{tab:levels} summarize.
At the highest rate the lead closes late in training, when the small batch reaches its failure point inside the budget.
\begin{figure}[H]
\centering
\includegraphics[width=\linewidth]{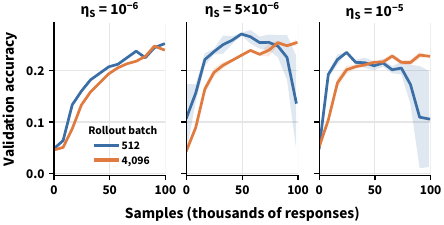}
\caption{The 512 and 4,096 batches at three learning-rate levels with $\mu{=}8$.
Color marks the rollout batch, the larger batch uses $\sqrt8$ times the smaller batch's rate $\eta_S$, and bands span two seeds where available.
The small batch's lead grows with the rate, and at the highest rate it closes late, when the small batch starts to fail.}
\label{fig:levels}
\end{figure}

\paragraph{Conservative regime.} At $\eta_S{=}10^{-6}$ fewer than $2{\times}10^{-4}$ of tokens are clipped, and none at $k{=}0$, so any gap that appears here cannot come from the clip.
One appears anyway.
With $\mu{=}1$, where the clip cannot act at all, the large batch lags the small one by 0.022 and 0.021 in the two seed pairs and needs 35--50\% more samples to reach 15\% accuracy (left panel of Fig.~\ref{fig:ctrl}), about four times the $\Dauc{=}0.005$ that separates two seeds of the same configuration.
At $\mu{=}8$ the lag is similar ($\Dauc{=}0.016$, with 30k against 22k samples to 15\% accuracy).

\begin{figure}[H]
\centering
\includegraphics[width=\linewidth]{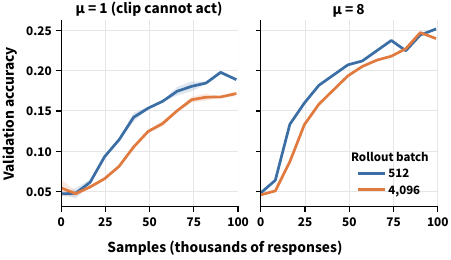}
\caption{The 512 and 4,096 batches in the conservative regime, $\eta_S{=}10^{-6}$, with $\mu{=}1$ (left, two seeds) and $\mu{=}8$ (right, one seed).
Color marks the rollout batch, and bands span the two seeds.
The larger batch lags in both, although less than $2{\times}10^{-4}$ of its tokens are clipped.}
\label{fig:ctrl}
\end{figure}
\paragraph{Clipped fraction by step.} Lemma~\ref{lem:drift} predicts that clipping builds up over the steps of a rollout, and faster for the large batch.
Figure~\ref{fig:perk} shows both.
The clipped fraction rises with $k$ at every learning rate, and the large batch's excess grows from 1.1--1.6 times the small batch's at $k{=}1$ to about 2--4 times from $k{=}3$ on, so the difference between batches accumulates within a rollout rather than being present from the first step.

\begin{figure}[H]
\centering
\includegraphics[width=\linewidth]{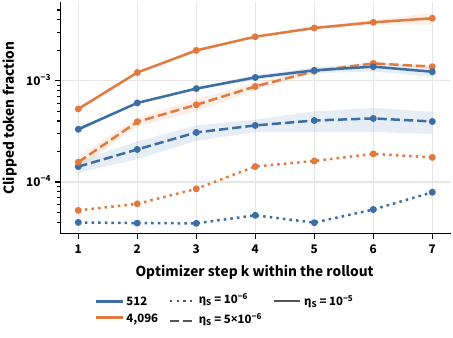}
\caption{Clipped token fraction at each optimizer step $k$ within a rollout for the 512 and 4,096 batches at three learning-rate levels with $\mu{=}8$, averaged over the first 49k samples, before any of these runs fails.
Color marks the rollout batch and line style the learning-rate level, and bands span two seeds where available.
Nothing is clipped at $k{=}0$, both fractions then rise with $k$, and the large batch's is 1.1--1.6 times the small batch's at $k{=}1$ and 2--4 times from $k{=}3$ on.}
\label{fig:perk}
\end{figure}

\paragraph{The gap over training.} Table~\ref{tab:levels} condenses each pair of learning curves into a few numbers, and Figs.~\ref{fig:gapall} and~\ref{fig:n2t} show what lies behind them.
With $\mu{=}1$, where the clip cannot act, the gap persists throughout training.
With $\mu{=}8$ it opens early and widens with the learning rate, and at the active rates it closes or reverses only when the small batch begins to fail late.
The samples needed rise with the batch at every accuracy target, so the cost does not hinge on the 15\% threshold of Table~\ref{tab:levels}.

\section{Failure-Law Fits and Failure Points}
\label{app:onset}
\begin{table*}[t]
\centering\small
\setlength{\tabcolsep}{6pt}
\begin{tabular}{lccc}
\toprule
free fit & $\beta$ & $\gamma$ & $\beta-\gamma/2$ \\
\midrule
least squares (18 runs) & 0.24 [0.00, 0.52] & 0.64 [0.29, 0.96] & $-0.08$ [$-0.23$, 0.09] \\
Tobit (21 runs, 3 censored) & 0.47 [0.19, 0.72] & 0.72 [0.32, 1.04] & 0.11 [$-0.09$, 0.32] \\
\midrule
fixed form & rms factor & log-lik. & LR test \\
\midrule
$\sqrt b/\eta$ (Eq.~\ref{eq:law}, $C{=}0.044$) & 1.53 & $-15.3$ & $p\approx0.29$ \\
$b/\eta^2$ (random walk) & 2.32 & $-24.7$ & $p<10^{-4}$ \\
$1/\eta$ (batch-blind) & 1.68 & $-23.9$ & $p<10^{-4}$ \\
free (Tobit) & 1.57 & $-14.0$ & -- \\
\bottomrule
\end{tabular}
\caption{Fits of the failure law $N_{\mathrm{fail}}\propto b^{\beta}\eta^{-\gamma}$.
Brackets give 90\% bootstrap intervals, and the square-root rule moves the failure point by the factor $b^{\beta-\gamma/2}$.
The rms factor is the geometric rms residual of the failing runs, or the fitted noise scale for Tobit, the log-likelihood is censored, and the LR test compares each fixed form with the free Tobit fit.}
\label{tab:lawfit}
\end{table*}

\begin{table*}[t]
\centering\scriptsize
\setlength{\tabcolsep}{3pt}
\begin{tabular}{llccc@{\hspace{6pt}}c@{\hspace{6pt}}llccc}
\toprule
run & clip & onset rollout & onset $N$ & acc.\ halves && run & clip & onset rollout & onset $N$ & acc.\ halves \\
\cmidrule(r){1-5}\cmidrule(l){7-11}
S $10^{-5}$ seed 0 & on  & none$^{\dagger}$ & -- & -- && S $10^{-5}$ ($\mu{=}16$) & \textbf{off} & none & -- & 34k \\
S $10^{-5}$ seed 1 & on  & 153 & 78k & 83k && M $1.7{\times}10^{-5}$ seed 1 & \textbf{off} & 20 & 31k & 33k \\
S $5{\times}10^{-6}$ seed 0 & on & 187 & 96k & -- && M $10^{-5}$ & \textbf{off} & 48 & 74k & 81k \\
S $5{\times}10^{-6}$ seed 1 & on & 156 & 80k & 92k && M $10^{-5}$ seed 1 & \textbf{off} & 24 & 37k & 48k \\
S $1.7{\times}10^{-5}$ & on & none & -- & 71k && L $2.8{\times}10^{-5}$ s0/s1 & on & none in 24 & $>$98k & -- \\
S $2{\times}10^{-5}$ & on & none & -- & 36k && L $1.4{\times}10^{-5}$ s0/s1 & on & none in 24 & $>$98k & -- \\
S $2.8{\times}10^{-5}$ & on & none & -- & 38k && L $2.8{\times}10^{-5}$ seed 0 & \textbf{off} & 7 & 29k & 29k \\
L $5.7{\times}10^{-5}$ & on & none & -- & 87k && L $2.8{\times}10^{-5}$ seed 1 & \textbf{off} & 10 & 41k & 48k \\
S $10^{-5}$ seed 0 & \textbf{off} & 81 & 41k & 46k && L $1.4{\times}10^{-5}$ & \textbf{off} & none in 15 & $>$61k & -- \\
S $10^{-5}$ seed 1 & \textbf{off} & none in 72 & $>$37k & -- && L $10^{-5}$ seed 0 & \textbf{off} & none in 36 & $>$147k & -- \\
M $1.7{\times}10^{-5}$ & \textbf{off} & none in 40$^{\S}$ & -- & 51k && L $10^{-5}$ seed 1 & \textbf{off} & 8 & 33k & 29k$^{\dagger\dagger}$ \\
S $10^{-5}$ seed 1$^{\natural}$ & \textbf{off} & 61 & 31k & -- && L $5.7{\times}10^{-5}$ & \textbf{off} & none & -- & 21k \\
S $10^{-5}$ + grad-clip 1.0 & \textbf{off} & 68 & 35k & 36k && XL $4{\times}10^{-5}$ & \textbf{off} & none in 10 & $>$82k & --$^{\flat}$ \\
S $5{\times}10^{-6}$ & \textbf{off} & 145 & 74k & -- && DeepScaleR L $2.8{\times}10^{-5}$ & \textbf{off} & 10 & 41k & 52k \\
S $2{\times}10^{-5}$ & \textbf{off} & 35$^{\sharp}$ & 18k & 19k && 7B S $2{\times}10^{-6}$ & \textbf{off} & 81$^{\sharp}$ & 41k & 45k \\
S $5{\times}10^{-6}$ seed 1 & \textbf{off} & 82 & 42k & 45k && 7B L $5.7{\times}10^{-6}$ & \textbf{off} & none$^{\dagger\dagger}$ & -- & 53k \\
S $1.7{\times}10^{-5}$ & \textbf{off} & none & -- & 29k && 7B S $10^{-6}$ & \textbf{off} & none in 192 & $>$98k & -- \\
S $2.8{\times}10^{-5}$ & \textbf{off} & 17 & 9k & 13k && L $2.8{\times}10^{-5}$ s0 & $\varepsilon{\times}\sqrt8$ & none$^{\ddagger}$ & -- & -- \\
S $10^{-5}$ ($\mu{=}4$) & \textbf{off} & 197 & 101k & 86k && L $2.8{\times}10^{-5}$ s1 & $\varepsilon{\times}\sqrt8$ & 22 & 90k & 93k \\
S $10^{-5}$ ($\mu{=}4$) seed 1 & \textbf{off} & 45 & 23k & 26k && & & & & \\
\bottomrule
\end{tabular}
\caption{Failure points of every run in the active regimes under both parts of the failure criterion (\S\ref{sec:setup}).
The columns give the sustained gradient-norm onset in rollouts and in samples, and the sample at which accuracy first halves from its best value, with -- for runs that never halve, and clip $\varepsilon{\times}\sqrt8$ is the widened radius.
Marks denote a decay from the peak without a blow-up ($^{\dagger}$), a decay from 57k ($^{\ddagger}$), a quiet collapse in which accuracy falls while the gradient norm stays near baseline except for one spike ($^{\S}$), a second seed rerun on different hardware ($^{\natural}$), a 40\% fall in accuracy from 41k without halving ($^{\flat}$), a collapse followed by recovery ($^{\dagger\dagger}$), and a KL blow-up with every response at the length cap and the gradient exactly zero afterwards, rather than a gradient-norm excursion ($^{\sharp}$).}
\label{tab:onset}
\end{table*}

Table~\ref{tab:onset} lists both failure criteria for every run in the active regimes.
Using three or five times the early median instead of four shifts the detected gradient-norm onset by at most four rollouts, except for S at $5{\times}10^{-6}$ (rollouts 151, 156 and 165 for seed 1, while seed 0's excursion at rollout 187 does not reach $5\times$).
Onsets are counted in rollouts, each of $\mu$ optimizer steps.
Some runs fail without a gradient-norm excursion: accuracy falls while the gradient norm remains near baseline, or the KL divergence between the updated policy and the rollout policy blows up, after which every response hits the length cap and the gradient is exactly zero.
Two large-batch runs collapse and then recover.
The accuracy criterion detects these cases, so the failure definition in \S\ref{sec:setup} gives it precedence.

\paragraph{Fitting the failure law.} We regress $\log N_{\mathrm{fail}}$ on $\log b$, with $b$ the responses per optimizer step, and on $\log\eta$, in two ways, and report both in Table~\ref{tab:lawfit}.
Least squares uses the eighteen runs that fail.
A right-censored normal likelihood (Tobit) adds the three runs that never fail, whose budget is a lower bound on their failure point.
Intervals are 90\% percentile bootstrap intervals over runs (2000 resamples).
We also hold the exponents at three fixed forms and compare them with the free Tobit fit by a likelihood-ratio test with two degrees of freedom.
The drift form of Eq.~\eqref{eq:law} passes the test ($p\approx0.3$), whereas the random-walk and batch-blind forms fail it ($p<10^{-4}$).
Because single failure points vary across seeds, as Table~\ref{tab:onset} shows, our conclusions rest on the aggregate scaling across 21 runs rather than on any single pair.

\section{Clipping Analyses}
\label{app:clip}

\subsection{Replay}
\label{app:replay}
\begin{figure}[!ht]
\centering
\includegraphics[width=\linewidth]{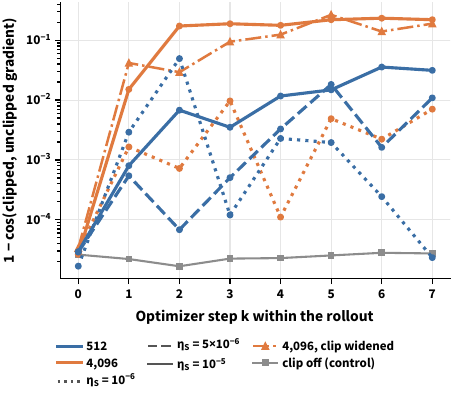}
\caption{Replayed rotation of the update by the clip, measured as one minus the cosine between the clipped and unclipped gradients, at each optimizer step $k$ within a rollout with $\mu{=}8$ (median over checkpoints).
Color marks the rollout batch and line style the learning-rate level, the dash-dot line is the 4,096 batch with a widened clip, and the gray line is the clip-off control, which marks the numerical floor.
At $k{=}0$ every ratio is one and the clip cannot act, and with $\eta_S{=}10^{-5}$ the large batch's rotation is about 0.2 from $k{=}2$ on, 6 to 50 times the small batch's.}
\label{fig:replay}
\end{figure}
Figure~\ref{fig:replay} summarizes the replays of the final checkpoint of each $\mu{=}8$ run at its training learning rate.
Each replay starts from a fresh AdamW state for both batches, so comparisons between batches are matched, although the fresh state makes the first updates somewhat larger, and absolute clipped fractions run about twice those in the training logs.
As a numerical calibration, removing the clip gives $\dclip\approx2$--$3{\times}10^{-5}$ at every $k$, which we treat as the measurement floor.
At the intermediate rate, the small batch's $\dclip$ is mostly $10^{-4}$ to $10^{-2}$ with isolated spikes.
In one mini-batch where only $7{\times}10^{-5}$ of tokens are clipped, $\dclip$ reaches 0.72, and Proposition~\ref{prop:share} then implies that these few tokens carry a gradient component comparable in norm to the full mini-batch gradient, a small-batch instance of what the large batch shows at every step.

\paragraph{Gradient-noise scale.} At the same checkpoints, with the unclipped objective, the squared norm of the gradient averaged over a whole rollout is within 4\% of the pure-noise prediction from four-response micro-batches.
The signal is therefore not resolvable against the noise, and the simple noise scale is above $6{\times}10^{4}$ responses at every checkpoint after the first few steps, far above both optimizer batches (64 and 512).
This supports assumption (A1) of \S\ref{sec:method}.
The implied critical batches of Remark~\ref{rem:bcrit} lie well below this noise scale, as Eq.~\eqref{eq:bstar} allows.

\subsection{Widened clip radius}
\label{app:fix}
\begin{figure}[!ht]
\centering
\includegraphics[width=\linewidth]{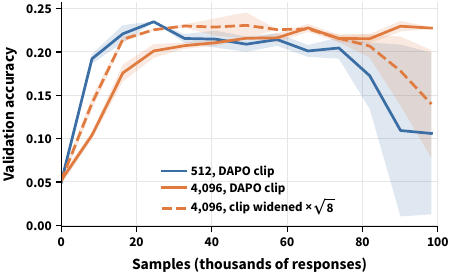}
\caption{Widened clip for the 4,096 batch with $\eta_S{=}10^{-5}$ and $\mu{=}8$, against the 512 batch.
Color marks the rollout batch and line style the clip radius, solid for the DAPO radius and dashed for the log-space radius multiplied by $\sqrt8$, and bands span two seeds.
The widened radius gives the large batch the small batch's shape, with an earlier rise and then decay, and one seed collapses.}
\label{fig:fix}
\end{figure}
We multiplied the log-space clip radius of L at $2.8{\times}10^{-5}$ by $\sqrt8$, giving $(\varepsilon_{\mathrm{low}},\varepsilon_{\mathrm{high}})=(0.468,1.01)$.
This brings its clipped fraction down to the small batch's level ($4$--$7{\times}10^{-4}$ instead of $2$--$4{\times}10^{-3}$).
Both seeds learn faster early (15\% accuracy after 10k instead of 14k samples) and then decay: seed 0 from about 57k samples to a final 0.202, below the plain L's 0.23, and seed 1 collapses at 90k samples, ending at 0.078, as Fig.~\ref{fig:fix} shows.
Replaying seed 0's checkpoints explains this.
The wider radius halves the clipped fraction but lets the ratios drift further (99th percentile of $|\log\rho|$ at $k{=}7$: 0.66--0.75 against 0.45--0.51), and the tokens that are still clipped carry the same share of the gradient (0.47--0.53 against 0.50--0.55).
Widening the clip moves the large batch toward the small batch's behavior: it learns faster early, but loses much of its stability advantage, and one seed eventually collapses.

\section{Checkpoint Robustness}
\label{app:instruct}\label{sec:secondmodel}
The progress result carries over to Qwen2.5-1.5B-Instruct, the general-purpose checkpoint of the same family: the small batch leads early at both $\mu{=}1$ and $\mu{=}8$, although this checkpoint trains in a markedly different regime.
At a matched learning rate, its per-step KL is about ten times larger, its within-rollout log-ratio spread is 2.7 times larger, and its entropy is 0.5--0.8 rather than about 0.2.
At one step per rollout both batches learn at $5{\times}10^{-6}$, the small batch leading over the first 30k samples.
At $\mu{=}8$ a 24-rollout ladder puts its stable window at $1$--$3{\times}10^{-6}$, an order of magnitude below the math model's.
At $2{\times}10^{-6}$ S is stable for its whole budget, peaks at 0.102 after 25k samples and ends at 0.082.
Its L partner is also stable but much slower: $0.058$, $0.075$ and $0.076$ at 16k, 49k and 98k samples against $0.088$, $0.102$ and $0.082$, as the right panel of Fig.~\ref{fig:inst} shows.
The small batch leads at every evaluation, reaches 8\% accuracy after 14k samples against 86k, and $\alpha{=}0.67$.

\begin{figure}[!ht]
\centering
\includegraphics[width=\linewidth]{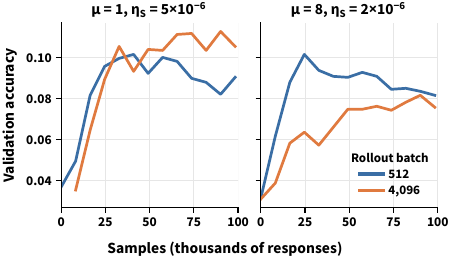}
\caption{Qwen2.5-1.5B-Instruct with one seed per run, comparing the 512 and 4,096 batches with $\mu{=}1$ and $\eta_S{=}5{\times}10^{-6}$ (left) and with $\mu{=}8$ and $\eta_S{=}2{\times}10^{-6}$ (right).
Color marks the rollout batch.
The small batch leads early in both settings.}
\label{fig:inst}
\end{figure}

\section{Batch-Ramp Switch-Point Sweep}
\label{app:sweep}
\begin{figure}[!ht]
\centering
\includegraphics[width=\linewidth]{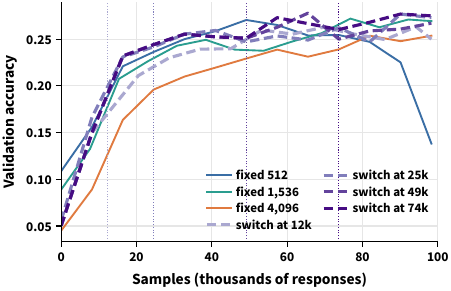}
\caption{Switch-point sweep with $\eta_S{=}5{\times}10^{-6}$ and $\mu{=}8$, in which one trajectory of the 512 batch (seed 0) switches to the 4,096 batch after 12k, 25k, 49k or 74k samples.
Dashed lines are the switched runs, darker for later switches, solid lines the fixed 512, 1,536 and 4,096 batches (seed means), and dotted verticals the switches.
Switches at 25k, 49k and 74k perform similarly, and only the earliest, at 12k, gives up part of the small batch's lead.}
\label{fig:sweep}
\end{figure}
Figure~\ref{fig:sweep} varies only the switch point along a common small-batch trajectory, which separates the timing of the ramp from variation in first-stage training.
Each second stage starts from the corresponding first-stage checkpoint with a fresh optimizer state and a one-rollout warm-up, and its first evaluation is within 0.016 accuracy of the trajectory before the switch.
The ramp is robust to the switch point: switches from 25k to 74k samples perform similarly, provided the switch precedes failure, and only the earliest switch, at 12k, gives up part of the small batch's progress advantage.

\end{document}